\documentclass[twoside,11pt]{article} 
\usepackage{jmlr2e}

\ShortHeadings{}{}
\firstpageno{0}
\title {IMPACT: Iteratively Moderated Probably Approximately Correct Teaching}
\author{\name Carl Trimbach \email ctrimbac@cs.brown.edu\\ 
\name Michael Littman \email mlittman@cs.brown.edu\\ 
\addr Authors' address line one\\ 
Authors' address line two\\ 
Authors' address line three}

\usepackage{amsmath, amssymb}
\usepackage{algorithm, algorithmic}
\usepackage{graphicx}

\newcommand{\prob}[1]{\text{Pr}\left(#1\right)}

\newtheorem{claim}{Claim}

\begin{document}

\begin{center}
\Large \textbf{Teaching with IMPACT} \\~ \\
\normalsize Carl Trimbach\quad ctrimbac@cs.brown.edu\\
Michael Littman\quad mlittman@cs.brown.edu\\
\end{center}
\vspace{20pt}

\begin{abstract}
Like many problems in AI in their general form, supervised learning is computationally intractable. We hypothesize that an important reason humans can learn highly complex and varied concepts, in spite of the computational difficulty, is that they benefit tremendously from experienced and insightful teachers. This paper proposes a new learning framework that provides a role for a knowledgeable, benevolent teacher to guide the process of learning a target concept in a series of ``curricular'' phases or rounds. In each round, the teacher's role is to act as a moderator, exposing the learner to a subset of the available training data to move it closer to mastering the target concept. Via both theoretical and empirical evidence, we argue that this framework enables simple, efficient learners to acquire very complex concepts from examples. In particular, we provide multiple examples of concept classes that are known to be unlearnable in the standard PAC setting along with provably efficient algorithms for learning them in our extended setting. A key focus of our work is the ability to learn complex concepts on top of simpler, previously learned, concepts---a direction with the potential of creating more competent artificial agents.
\end{abstract}

\section{Introduction}
In statistical learning theory, the PAC (probably approximately correct) model~\citep{valiant1984theory} has been a valuable tool for framing and understanding the problem of learning from examples. Many interesting and challenging problems can be analyzed and attacked in the PAC setting, resulting in powerful algorithms. Nevertheless, there exist many central and relevant concept classes that are provably unlearnable this way. Some of these learning problems are faced by human beings trying to accomplish everyday tasks, so declaring them to be too hard to learn is both counterintuitive and counterproductive. We contend that human learning of difficult tasks is greatly aided by the fact that other people have already mastered these same tasks and can organize the learner's experience in a way that simplifies the learning problem and restores tractability. 

In the standard PAC setting, a concept class is considered learnable if there exists an algorithm that, provided with bounded time and training experience, can output a classifier that has low error with high probability. This definition of ``learnability'' is a good measure of the complexity of a particular concept. However, there exist concepts that, under this strict definition of learnability, are provably impossible to learn given widely accepted cryptographic assumptions~\citep{kearns1994cryptographic}. Rather than conceding that such problems are too difficult to learn, one can analyze the problem under an altered model of learnability.

While there have been various proposed augmentations to the standard PAC learning setting, one very general model is Vapnik's privileged information model~\citep{vapnik2015learning}. In this setup, the training experience presented to the learner has standard domain data, $x$,  and labels, $y$, but additionally includes a vector of privileged information $\hat{x}$ for each training example. This supplementary information can be used by the learner during training time to make the task of choosing a classifier simpler by helping to highlight patterns in the data. In this setup, the training experience is not altered in any other way. At testing time, the learner only sees the data $x$ and must provide labels according to some chosen hypothesis. This model is particularly appealing because it deviates from the standard setting only during training time. It is not hard to imagine situations where advice, insights, or other assistance from a benevolent teacher helps a learner grasp a concept. The algorithm is evaluated in exactly the same way and by the same standards as in the standard PAC setting, though. For some prescribed input size, classification error, and confidence, we expect the algorithm to succeed within a bounded amount of time and with a bounded number of experiences.

One challenge presented by Vapnik's research is to find the types of privileged information that make certain problem classes learnable. Our work does so for multiple classes of problems that all have an underlying compositional structure. All of the models discussed in this paper, both prior work and our research, are what we refer to as iterative learning models. That is, the learning process relies on presenting the learner with new experience in a series of rounds. The structure of the iterative rounds themselves can be viewed as a specific type of privileged information. For example, imagine trying to capture the information available in a round-based training scheme with $k$ rounds to a learner in a single batch. The teacher can augment this batch of data with a $k$-bit privileged information vector where the $i^{th}$ bit of the string is $1$ if and only if that example would have been presented in round $i$, and $0$ otherwise. We will view the algorithms we discuss as iterative algorithms rather than privileged information algorithms, knowing they could be converted to the privileged information setting in this straightforward way. We will show that this ``round information'' is necessary and sufficient to learn certain concept classes. That is, for classes that are known to be unlearnable without this added information, we will present an algorithm that takes advantage of this information to learn efficiently. We will prove theoretically and verify empirically that our iterative learning algorithm can efficiently learn these otherwise unlearnable concepts.

\section{Related Work}

The concept of introducing a teacher into the standard learning framework is well studied. One way of categorizing different approaches in this area is by the power given to the teacher in the model. The most na\"{\i}ve approach, which makes the teacher maximally powerful, is to allow the teacher to generate samples with which to train the agent. Such a  model is a complete departure from the PAC setting, since there is no longer a distribution over training examples. This model renders the learning problem trivial in many domains, since the teacher can directly convey exactly what the agent needs to know. That is, if the goal is to learn some sort of function representable by a tree (as in the problems we consider), the teacher can generate $O(1)$ examples for each node in the tree to show the agent the important bits and how they are related. The entire function can be taught with sample complexity $O(k)$ where $k$ is the number of nodes in the function's tree representation. This problem is not one that we are concerned with in our work, as the absence of a distribution on the domain means we will have no means to compare to learning in the standard PAC setting.  More importantly, since the goal of this work is to create a model that works for different learning problems, we do not want to assume that even the most capable of teachers has the ability to devise and construct helpful training examples. The reliance on the environment to provide the agent with experience lends itself to many real world applications.

A possible next logical step in reducing the teacher's power would be to force the teacher to draw examples from a distribution for training and testing the agent. This brings the problem more in line with the standard PAC setting. An attempt at iterative learning with this type of teacher was proposed using a hierarchical learning algorithm presented by \citet{rivest1994formal}. This work attempts to solve the problem of learning Boolean formulae by using a teacher who cannot create examples, but can still edit their labels before presenting them to the learner during the training rounds. They show how such a benevolent teacher can guide a learner to a hypothesis that is ``reliable'' (makes no incorrect classifications), and with high probability is ``useful'' (doesn't return ``I don't know'' too often). Because the agent is allowed to refuse to classify some instances, this algorithm does not fit into the PAC framework in the strict sense. We will show later that we can still meaningfully compare their algorithm to our results, however. 

The more troubling concern with this setup is the potency of the teacher, because the teacher is allowed to alter the training examples' labels during training. When this power is restricted to being used for this particular problem and within the context of the algorithm the authors provide, there is no issue. However, since the goal of our work is to provide a new learning framework for a broad class of problems solved in this iterative way, we cannot allow the teacher to have so much power. Doing so would leave the model open to allowing collusion between the teacher and the learner leading to the design of brittle learning algorithms.
Further, it may not always be the case that even the most helpful of teachers has the ability to alter the reality of the domain and relabel examples. Formal constraints for enforcing ``collusion-free'' interactions between teachers and learners have been addressed~\citep{zilles2011models}. Indeed, the setup of \citet{rivest1994formal} does not qualify as a valid, collusion-free teacher--learner pairing, because it violates the rule that says a teacher must present a set of training samples that is consistent with the concept $c$ being taught.

To further combat this potential for collusion, we restrict the capabilities of the teacher. The restrictions we put in place bring our model in line with the criteria described by \citet{zilles2011models}. In the setting we propose, the teacher is still bound to drawing examples from the testing and training distribution. The teacher cannot alter any labels of the examples drawn from the domain distribution at training or testing time nor can the teacher and learner agree upon an ordering over the sample space. The only capability the teacher has is to disallow certain examples from being presented to the learner. The intuition here is that if an example is too confusing or uninformative, the teacher can block the agent from seeing it. Because this example was drawn from the distribution and evaluated by the teacher, it still counts towards the sample complexity calculation, even though the agent never gets the opportunity to learn from it.

This idea of limiting the data seen by the agent is used in other settings as well. For example, in knowledge distillation \citep{hinton2015distilling}, a subset of training data is used to train a distilled, smaller deep net after a more cumbersome model was already trained. In this setting however, the authors are reusing the training examples to get a more efficient model, whereas our work seeks to limit training data to get a more efficient training algorithm. Additionally, this work falls victim to the same trap of relabeling examples that was discussed previously.

\section{Problem Descriptions}
\begin{figure}
\centering
\begin{minipage}{0.15\textwidth}
\includegraphics[width = 0.8\textwidth]{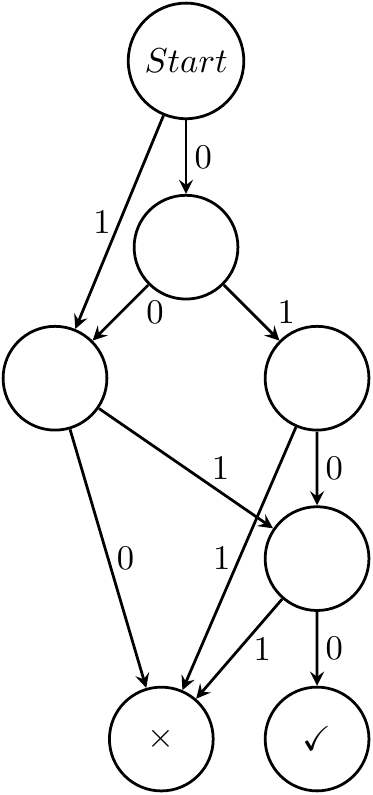}
\end{minipage}%
\begin{minipage}{0.15\textwidth}
\includegraphics[width = 0.9\textwidth]{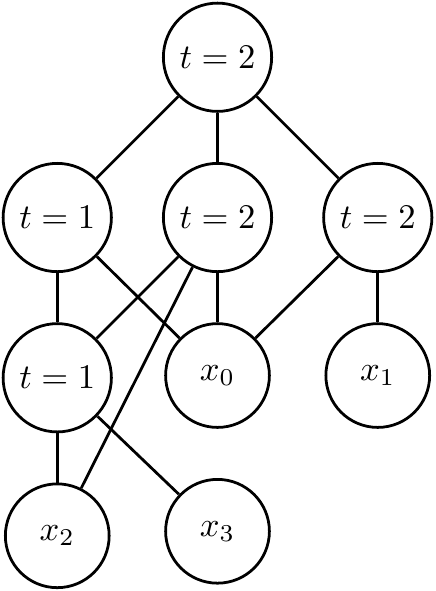}
\end{minipage}%
\begin{minipage}{0.2\textwidth}
\includegraphics[width = 0.9\textwidth]{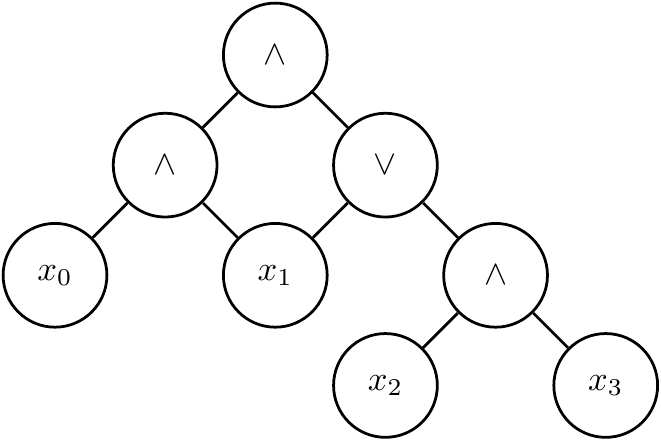}
\end{minipage}%
\caption{The left is the DAG representing an ADFA on Boolean strings with an edge for every possible character, Start, Accept (checkmark), and Reject (X) nodes. The center diagram gives an example threshold circuit where each node is labeled either as an pure input, or as a threshold gate with a threshold of $t$. Finally, at the right is an example DAG representing a Boolean formula. Base nodes are labeled with the pure inputs and internal nodes are labeled with the operator that defines their value.}
\end{figure}

The problems we discuss were chosen not for their independent interest, but because they are all of the problems that are provably unlearnable in the standard PAC setting given by ~\citep{kearns1994cryptographic}.

The first problem, \emph{Boolean Formulae}, is the problem in which the learner is tasked with learning a Boolean formula over $n$ inputs. The domain is $\mathcal{X} = \{0,1\}^n$, $n$-bit strings, and the concept class, $\mathcal{C}_{BF}$, is comprised of all Boolean functions over $n$ inputs that can be represented by Boolean circuits with $O(p(n))$ nodes, where $p(n)$ is some polynomial in $n$. We focus on polynomially sized Boolean circuits, since larger circuits would take too long for an algorithm to describe, let alone learn. Since we only require the algorithm to learn a Boolean formula; it need not learn the exact structure of the circuit representation as long as the concept it chooses is functionally equivalent to the actual formula over the input distribution.

The \emph{Threshold Circuits} problem again has domain $\mathcal{X} = \{0,1\}^n$. The concept class $\mathcal{C}_{TC}$ consists of all constant depth circuits of polynomially many threshold gates in the number of input bits $n$. Each threshold gate consists of a set of $k$ inputs (either input bits or outputs of other threshold gates in the circuit) and will output a $1$ if and only if at least $t$ of the $k$ inputs is $1$, and a $0$ otherwise. This threshold value $t$ is specific to the gate. The goal of the learner here is to learn the circuit (or an equivalent circuit).

Finally, the \emph{Acyclic Deterministic Finite State Automata} (ADFSA) problem involves learning concepts from the class $\mathcal{C}_{\rm \small \it ADFSA}$ of acyclic deterministic finite automata that can recognize strings of length at most $n$. These automata have a single accept and reject state and have a their total number of polynomial in the input size, $n$. For simplicity, we consider $\mathcal{X} = \{0,1\}^n$, the set of all Boolean strings of length $n$, but other alphabets are handled similarly. Again, we only require that the learner decide upon a ADFSA that is functionally equivalent to the target ADFSA over the input distribution.

\section{A Framework for Learning from a Teacher}
The Iteratively Moderated PAC Teaching (IMPACT) model is a generalization of the standard PAC learning model. In the PAC learning model, a problem is defined by a domain $\mathcal{X}$ of possible instances, a concept class $\mathcal{C}$ that consists of functions mapping $x\in \mathcal{X}$ to labels $y\in\{0,1\}$, and a distribution $\mathcal{D}$ over $\mathcal{X}$ from which a set $S = \{x_i\in\mathcal{X}|1\le i\le m\}$ of $m$ training examples is drawn. The learner is tasked with finding a hypothesis $h$ from the hypothesis class $\mathcal{H}$ (mapping $\mathcal{X}$ to binary labels much like the concept class $\mathcal{C}$). The goal is for $h$, with probability $1-\delta$, to have at most $\epsilon$ expected error on testing data drawn from the distribution $\mathcal{D}$. We say a problem is PAC learnable if, for any $\epsilon, \delta >0$, there is an algorithm $\mathcal{A}$ that finds such an $h$ in time (and number of samples) polynomial in $\frac{1}{\delta}$, $\frac{1}{\epsilon}$, and parameters capturing the size of the problem description.

In our model, learning instead takes place in a number of rounds, $R$. In each round, the teacher selects $S_r\subseteq S$, a subset of the training sample, to show the learner. The learner uses these samples to learn some simple concept $c_r\in \mathcal{H},\ 1\le r\le R$. We require that $\mathcal{H}$ be PAC learnable within the context of the round $r$. It is important to note it need not be the case that $\mathcal{H} =\mathcal{C}$. In fact, if $\mathcal{H} = \mathcal{C}$, and if we set $R=1$, then we revert to the standard PAC model. Instead, at the end of each round, the learner can choose to augment its representation of the attribute space in some way to carry information from the previous rounds into future rounds. Incorporating simple concepts from previous rounds allows the learner to gradually build upon knowledge for use in later rounds. This feature is the critical addition to the standard framework that allows the teacher to teach more effectively. The ``simple'' concept class $\mathcal{H}$ remains the same from round to round except that its domain changes to account for the information carried forward from previous rounds.

Additionally, we allow no communication between the learner and the teacher once the learning process has started. The teacher is not allowed to alter or order the examples in any way, except to filter out potentially irrelevant or confusing ones at each round. The learner only sees the set of examples the teacher has allowed, and has no other indication of how many or which examples were moderated. Similarly, the learner cannot communicate back to the teacher what its current hypothesis or augmented attribute space looks like. These constraints prevent the teacher from tailoring future example sets to the learner's progress.
\section{Algorithms and Analyses}  \label{algorithms}

\begin{figure}
\begin{algorithmic}
\floatname{algorithm}{Procedure}
\renewcommand{\algorithmicrequire}{\textbf{Input:}}
\renewcommand{\algorithmicensure}{\textbf{Output:}}
    \REQUIRE set of initial attributes $Z_0$, learner's hypothesis space $\mathcal{H}$, true concept $c\in \mathcal{C}$, examples $S$, number of rounds $R$
    \ENSURE learned concept $h_R\in \mathcal{H}$
    \FOR{$r \gets 1 ,\ldots, R$}
	\STATE $S_r\gets$ \textbf{ModerateSample}($S,r$) \COMMENT {Filter $S$ for round $r$.}
	\STATE $h_r \gets$ \textbf{PACLearn}($Z_{r-1},S_r,\mathcal{H}$)
	\STATE $Z_r \gets$ \textbf{augmentAttributes}($Z_{r-1},h_r$)
    \ENDFOR
\end{algorithmic}
\caption{Pseudocode for the general algorithm.}
\label{alg:general}
\end{figure}
We now present algorithms that solve the problems from previous sections in the IMPACT model. The first, for Boolean Formulae, is the most straightforward. It is followed by the IMPACT algorithm for Threshold Circuits. Lastly, we present a similarly structured algorithm to teach Acyclic Deterministic Finite State Automata. 

To fully describe an algorithm in the IMPACT framework, we must define the setup of the round structure and the moderating rule the teacher uses in each round. Additionally, we must define the simple hypothesis space $\mathcal{H}$, learning algorithm, and attribute space augmentation used by the learner in each round. In all of the examples in this paper, we will follow the same basic algorithm structure, given in Figure~\ref{alg:general}. The rounds will be defined by individual nodes in a directed acyclic graph (DAG) representation of the function to be learned. The teacher will, at each round, moderate the set of samples so that the subset passed to the learner makes sense in the context of the problem. The specifics of how the filtering is done will be described below. Additionally, the hypothesis space and algorithm used by the learner will be problem dependent, but the learner will always augment the attribute space with some version of the learned concept.

The similarity in the structure of these algorithms leads to proofs that follow a general format as well. All of the algorithms rely on the fact that the concept being learned can be represented as a DAG whose size is polynomial in the input size of the problem. In each of the algorithms, we note that the teacher is iteratively moderating rounds of learning according the DAG structure inherent in the problem. The first component of each proof, then, will be to bound the error of the learned hypothesis at each node of the DAG with high probability.

One source for potential concern is that during the training procedure, the teacher is filtering away examples from the data set and therefore, the learner is learning on a different distribution than that which truly exists in the environment. Potentially, this negates any possibility of accurate performance bounds, since the training and test distributions are not the same.

To help argue that the learner's performance is preserved even in the full distribution, we introduce the concepts of \textit{relevance} and  \textit{correlation}. Both are properties of a single example at a particular node in the computation DAG. We say that an example $x$ is relevant at node $i$ if short circuiting and inverting the signal at the output of node $i$ causes the value at the root to invert as well. That is, it tells us whether the output of node $i$ is important in calculating the value at the root of the DAG, given that all other inputs are set according to $x$. Irrelevant examples at node $i$ have the same value at the root regardless of $i$'s output and are determined completely by the calculations made in the other parts of the DAG.  Intuitively, the accuracy of node $i$ is only meaningful only for examples that are relevant at $i$, since any incorrect classification at $i$ would not affect the overall output at the root for irrelevant examples. Correlation also refers to the relationship between the signal at the output at node $i$ and the value at the root of the computation tree. If the values are the same, we say that $x$ is correlated at $i$. If the values are different, we say that $x$ is anticorrelated at $i$.

First, we observe that for all examples that are irrelevant to node $i$, the learner cannot make mistakes, so we do not count these examples as incorrect. We now argue that we can convert any computational DAG into a form that is easier to work with theoretically. 

\begin{claim}
We can restructure any computation DAG, $G$, to form a new computation DAG, $G'$ such that it's size is at most $|G'|\leq 2|G|$. Further, in $G'$, for all nodes $n$, for all examples $R_n$ where $R_n$ is the set of relevant examples at $n$, $R_n$ is either exclusively correlated or anticorrelated.
\end{claim}

The above claim is accomplished via a transformation we define here on $G$. Consider $G$ to be a computation DAG represented by leaves which are literals of the Boolean function $F$ and internal nodes which are $and$, $or$, or $not$. By deMorgan's law, we can propagate all the $not$ nodes to just above the literals. Doing so creates at most two versions of each node in the DAG, the original, and its compliment. An example of this process is shown in Figure~\ref{fig:demorgan}.

Relevant examples at node $n$, by definition, propagate their output signal to the root of the tree through unblocked $and$ and $or$ nodes. Since both non-blocking $and$ and $or$ nodes pass through whatever signal is applied to their unblocked input, correlation is determined entirely by the parity of number of negations on the path back to the root. That is, an even number of negations on the root-node path means relevant examples will be correlated and an odd number implies relevant examples will be anticorrelated. In the restructuring, since we push all negations to the literals, all internal nodes will have relevant examples that are always correlated.

This means we can restructure any computation DAG such that it meets the criteria for Lemma~\ref{lem:relevance} and therefore Theorem~\ref{thm:relevance}, as well. Since we have at most twice as many nodes in the DAG, we have increased the overall sample complexity (and computational complexity) of the problem by a factor of $2=O(1)$.

\begin{lemma}
\label{lem:relevance}
Given a computational DAG where all negation nodes reside just above the leaves, the set of all relevant examples at a node will always be exclusively correlated. The set of irrelevant examples at a node may be of mixed relevance.
\end{lemma}
\begin{proof}
Consider the path from node $i$ to the root of the tree. The signal from $i$ travels through $and$ and $or$ nodes. By definition, those nodes must be unblocked in all relevant examples, so the signal passes through. Since there are no $not$ nodes for the signal to pass through, all of the relevant examples must be correlated.
\end{proof}

\begin{theorem}
\label{thm:relevance}
Training on only relevant data does not degrade test performance on the full data set.
\end{theorem}
\begin{proof}
The teacher is only filtering examples. Therefore, the teacher is sampling from the distribution where the sum of the weight of irrelevant examples has been distributed uniformly across relevant examples. Since irrelevant data cannot be misclassified at a particular node, the overall probability of error can only decrease at test time. That is, the amount of weight on the examples on which the learner errs in the training distribution can only be at most that of the weight of misclassified examples during testing time. 
\end{proof}

Next, we will argue that if the individual rounds are learnable with a polynomially sized input (that is, the subsample provided by the teacher is required to only be polynomially sized), then the overall sample complexity for the problem remains polynomial as well. Finally, we will make an argument that, in the final round of learning, the error accumulated throughout the entire algorithm is bounded in a way that enables PAC bounds for the entire algorithm.

\subsection{Boolean Formulae}
To learn Boolean Formulae, the teacher represents the target formula as a polynomially sized DAG with branching factor $2$. Then, the teacher draws a polynomially large training sample, $S$. The teacher arranges rounds in postfix order (children before the parent). For each round, the teacher then partitions the examples into sets such that the label given at the target node is correlated or anti-correlated to the label given at the root (true label). Finally, the learner is presented with the larger of the two sample subsets. The learner then uses $\epsilon$-exhaustion~\citep{mitchell1997machine} over a hypothesis space of the basic Boolean operators (\emph{AND}, \emph{OR}) on all possible pairs of input attributes and their negations. Since we know that this hypothesis space is finite, and polynomially large, $\epsilon$-exhaustion will be efficient. The output of each of these rounds will be stored as an additional derived input.  To elegantly handle negation nodes (and account for the fact that we may be learning from anti-correlated examples), the learner will also store the complement of the hypothesis learned during each round. We will repeat this process for each node in the DAG until the learner learns the root node (true function). We will now prove that this algorithm is correct (within the standard PAC bounds) and efficient.

First, to argue that the learner has small error at any node, we note there are two node types in the DAG that can be the target for any given round. The first type, referred to henceforth as a base node, is a node whose children are only pure variables of the Boolean function, that is, not derived attributes. An internal node, then, is one that we define as having at least one child that is a base or internal node. The total number of internal and base nodes in the DAG representation will be denoted by $N$. As we mentioned previously, the DAG is polynomial in the size of the input, so $N=poly(n)$.

In the case of base nodes, we refer to classic PAC learning with a finite, realizable hypothesis class~\citep{mohri2012foundations} to argue that with high probability we can learn the relationship that governs the output of this node with sufficient accuracy, given sufficiently many relevant examples from the teacher. It is important to realize that since the learner is only attempting to choose two variables and their appropriate operand from potentially $O(N)$ variables, the size of the hypothesis space is bounded by
\[
|\mathcal{H}|=O\left({{N}\choose{2}}\right)=O(N^2)=O(poly(n)).
\]
Here, we use $N$ instead of $n$, since the learner may have access to some derived inputs at the time of learning the node, and it has the option of choosing from these as well.

At all internal nodes (including the root), learning is slightly more complicated. At these nodes, due to the possibility of errors in previous rounds of learning (that is, at nodes within the subgraphs rooted at the current target node), our derived input values may be erroneous. To higher level learning problems, these errors are experienced as corrupted versions of the derived inputs the algorithms need. So, to cleanly derive the bounds on learning error, we first introduce some notation. We denote the hypothesis returned by the learning algorithm in a particular round $r$ as $h_r$, and the hypothesis induced at a particular node by the overall formula as $h^*_r$---the target hypothesis. By construction, $h^*_r, h_r \in \mathcal{H}$. Here, $h_r(x)$ and $h^*_r(x)$ are meant to be the value that the respective hypothesis has on data assuming all nodes have been learned perfectly. That is, both $h_r$ and $h^*_r$ provide a formula for combining some other two nodes in the DAG thus far (either pure or derived), i.e. $x_2\wedge x_4$. Using that formula, we evaluate the hypothesis on the uncorrupted versions of the contributing nodes. Contrastingly, $h^*_r(\tilde{x})$ and $h_r(\tilde{x})$ have the formula of the respective hypotheses, but are evaluated on the possibly corrupted contributing nodes. It is crucial to our analysis that $h(x)$ and $h(\tilde{x})$ need not be equivalent if the contributing nodes are corrupted, for any $h\in \mathcal{H}$.

Because it is the case that these hypotheses might not have the same labels, we have a notation to express the difference between these hypotheses on similar inputs. We denote the difference between two (possibly the same) hypotheses on a distribution $D$ of inputs (either corrupted or pure) as
\[
||f(x)-g(\tilde{x})||_D=\textbf{E}_{x\sim D}[f(x)\neq g(\tilde{x})].
\]
This notation is understood to be the expected number of samples that will disagree when we evaluate $f$ on uncorrupted data and $g$ on corrupted data. Similarly, we denote the difference between two functions on a sample $S$ of inputs as
\[||f(x)-g(\tilde{x})||_S=\frac{1}{m}\sum_{i=1}^m 1\{f(x_i)\neq g(\tilde{x_i})\}.
\]
This notation is understood in the same way, except, here, the expectation is replaced by a sum over a sample, $S$, of inputs. The value tells us the fraction of samples on which the two disagree. Not surprisingly, these two values, when applied to $h_r(\tilde{x})$ and $h^*_r(x)$, correspond to our testing and training errors, respectively. They are a measure of the difference between the algorithm's hypothesis' classification of the potentially corrupted derived inputs it uses versus the true concept's classification of the uncorrupted input data.

To show a bound on the testing error of our chosen hypothesis, we argue that, at a particular node, for any $h\in \mathcal{H}$,
$||h(\tilde{x})-h(x)||_D\le \epsilon_L + \epsilon_R$,
where $\epsilon_L, \epsilon_R$ are the errors of the left and right children, respectively. Intuitively, any single hypothesis can only classify the same input differently if a contributing child node has corrupted the signal. In other words, $h(x)$ and $h(\tilde{x})$ can only differ where $x$ and $\tilde{x}$ differ. We know that $\tilde{x}$ only differs from $x$ where some contributing node made an incorrect classification, and therefore is bounded by $\epsilon_L+\epsilon_R$.

Hoeffding's inequality~\citep{mitzenmacher2005probability} gives us a high probability upper bound on the testing error at a node:
$$ 
\prob{\left| ||h_r(\tilde{x})-h^*_r(x)||_S - ||h_r(\tilde{x})-h^*_r(x)||_D \right|\ge \epsilon}
$$
\vspace{-1ex}
$$
\le 2e^{-2m\epsilon^2} \le \delta.
$$ 
Rearranging the above equation gives us a bound of $m\ge \frac{\ln(\delta/2)}{2\epsilon^2}$ required samples to learn the node. Since the teacher is partitioning the original sample into at most $2$ subsets as long as we have $2\frac{\ln(\delta/2)}{2\epsilon^2}$ examples to start with, we'll always have sufficient training data at a node, and the overall sample will remain polynomial in size.

Further, the above implies with probability $1-\delta$:
\begin{align*}
||h_r(\tilde{x})-h^*_r(x)||_D &\le ||h_r(\tilde{x})-h^*_r(x)||_S + \epsilon \\
    &\le ||h^*_r(\tilde{x})-h^*_r(x)||_S + \epsilon \\
    &\le \epsilon_L + \epsilon_R + \epsilon.
\end{align*}%
Since the above holds for any node in the DAG, it holds for the root. More importantly, we can recursively replace the error at each child by the error at the root of that subtree plus the errors of its children. Doing so for the entire DAG reveals that the error for the entire function (the root of the DAG) is the sum of the errors at every node in the DAG. 

Say we want to learn the function with error at most $\epsilon_T$. In this case, we would require that, at each node, the error $\epsilon_i$ be at most $\epsilon_T/N$ where $N$ is the number of nodes in the DAG. Substituting this quantity into our bound, at each node we need $ \frac{N^2\ln(\delta/2)}{2\epsilon_T^2}$ examples. Importantly, we note that this quantity is polynomial in $1/\delta,\ 1/\epsilon_T,$ and $n$, since $N=poly(n)$ by the problem definition. Thus, we have shown that we can bound the error at any node in the tree to be small with high probability, that the number of samples required to do so is polynomial in size, and that the error at the root of the tree remains small with high probability as well.

\subsection{Threshold Circuits}

The algorithm for Threshold Circuits is very similar to that of Boolean Formulae. One difference, however, is that the learner does not learn binary Boolean functions at each round. Instead, the learner learns a perceptron~\citep{rosenblatt1958perceptron}. We can set a perceptron with weights over all inputs (some of which may be set to $0$) that can fully capture the behavior of the gate. For example, say we have a threshold gate with $k$ inputs and a threshold of $t$. A perceptron can model this gate by putting weight $\frac{t}{k}$ on the $k$ relevant inputs, and $0$ elsewhere.

The teacher also uses a filtering strategy that is identical to the Boolean Formula algorithm. The teacher trains for one round at each gate in the circuit in postfix order. The teacher presents examples that are correlated (or anti-correlated) to the learner at every round.

In fact, our proof for Boolean Formulae learned from their respective circuits had no machinery that was specific to Boolean formulae. We only require that the learner's concept at every round is PAC learnable~\citep{blumer1989learnability}, that even after the teacher moderates the set of samples, we have enough data to learn efficiently, and that after the final round of learning, the learner returns a hypothesis that is accurate within the standard PAC bounds. Since linear separators are PAC learnable~\citep{blumer1989learnability}, we use the same technique to ensure that we always present at least half of the sample data for training, and we have just described how linear separators are sufficient to represent the behavior of threshold gates. Thus, the previous argument holds for the Threshold Circuit case as well.

\subsection{Acyclic Deterministic Finite State Automata}

At first glance, teaching ADFSA seems like an entirely different problem. There are issues related to defining accepting and rejecting states, and evaluating partial solutions at each round of teaching. However, with careful design of both the teacher's moderating algorithm and the learner's hypothesis space and algorithm, we can overcome these challenges and show that ADFSA are indeed learnable in the IMPACT model. To simplify the explanation, we consider ADFSA that branch only on binary inputs. A larger alphabet can be handled similarly.

Again, the teacher will teach nodes of the ADFSA in postfix order. However, we now have to be more careful about how to moderate the set of examples for the learner. Since strings of variable length (up to $n$) can be classified by this ADFSA, the teacher must be careful to not present confusing or ambiguous strings. That is, if the teacher is not careful, negative examples can be almost meaningless. Because the learner has no information about the path traversed in the ADFSA that has led to the target node, it has no way of knowing on which bit in each example the target node is evaluating the string. To rectify this situation, the teacher filters on multiple properties of the examples, while still only requiring an initial sample that is polynomial in size, detailed next. 

The teacher pulls a sample from the domain distribution. First, we want to ensure that the learner can evaluate all of the examples in the training sample from the same offset within the string. So, the teacher partitions the sample into (at most) $n$ subsets, depending on how long each string takes to arrive at the target node. Then, the teacher partitions each of these examples into correlated and anti-correlated subsets, for a total of at most $2n$ subsets. The teacher then chooses the largest among these subsets to pass to the learner. It is important to understand that examples that never touch the target node are also included, and they are placed in all $n$ subsamples. Their label at the target node can still be calculated by the teacher, who knows the full ADFSA, when deciding if they are considered correlated or anti-correlated data when evaluated at that subset's offset. 

At each round, the learner must create a new ADFSA node and choose where the two outgoing edges get connected. The attribute space for the learner begins with just an accepting and a rejecting terminal node. After each round, the learner augments the attribute space with its current hypothesis and the complement of the current hypothesis. The complement has identical structure, but any terminal accepting or rejecting nodes are swapped. Similarly to the Boolean Formulae case, the hypothesis space $\mathcal{H}$ is polynomially sized. Again, the learner can use $\epsilon$-exhaustion of the hypothesis space to choose the best hypothesis for a given node. Because the learner knows to expect aligned examples (something decided before the learning process begins), it can efficiently evaluate potential hypotheses in the  following way. For each possible offset of the input, evaluate all potential hypotheses, then choose the one that most closely matches the data. Since the inputs are of size at most $n$, and there are at most $N=poly(n)$ rounds of learning, this evaluation can be done exhaustively in polynomial time.

Again, the proof from Boolean Formulae holds for this case as well. The finite hypothesis space over which learning is taking place is different, but the structure of the problem and the algorithm remain the same. The only worry that remains is that, in this much more aggressive filtering scheme, we may require a huge initial sample to ensure that sufficiently many examples remain in the subsample passed to the learner. We know that the learner requires, at each node, $O(poly(n))$ examples. Also, our filtering scheme described above guarantees (by choosing the largest remaining subsample) that we pass a subsample whose size is at least $\frac{1}{2n}$, since the first step partitions into $n$ subsets, and then we further partition each of those into $2$ subsets. That means we need our overall sample size to be $2n\ poly(n)$, so that, after this partitioning, we still have sufficiently many samples. Since $2n\ poly(n) = O(poly(n))$, we still satisfy the sample-complexity constraint.
\begin{figure}
\begin{tabular}{| p{0.5\textwidth} | p{0.5\textwidth} |} \hline
\textbf{\citet{rivest1994formal}}		& \textbf{IMPACT} \\ \hline
algorithm to solve a problem for a particular domain	& learning framework for problems covering many domains \\ \hline
teacher can relabel examples	& teacher can only filter examples \\ \hline
reliable	& bounded mistakes (``usually reliable'') \\ \hline
useful	& perfectly useful (no ``don't know''s) \\ \hline
small sample complexity	& at most double the sample complexity \\ \hline
\end{tabular}
\caption{A table comparing our work with that of \citet{rivest1994formal}.}
\label{tab:compare}
\end{figure}
\subsection{Comparison to Previous Work}
We have just shown analytically how we achieve PAC type bounds on our learning problems using the IMPACT framework. To compare directly to \citet{rivest1994formal}, we can alter the type of learner used in our algorithm slightly. Using a learner that returns ``don't know'' whenever $\epsilon$-exhaustion doesn't result in a set of hypothesis that are all consistent, we can guarantee reliability. Further, our original algorithm only  makes a mistake when there are these same disagreements between remaining hypotheses. Therefore, our mistake bound still holds as a bound on usefulness. So, both our method and that of \citet{rivest1994formal} result in reliable, useful classifiers using an amount of training data that is polynomial in $n, \epsilon,$ and $\delta$. This proves that the additional power of allowing the teacher to alter training labels only gains a small reduction in sample complexity, and is unnecessary for learnability. A side-by-side comparison of the two works is shown in Figure~\ref{tab:compare}, for clarity.

\section{Experimental Results}
To demonstrate the practical implications of our approach, we conducted a series of experiments in which we taught from the class of polynomially sized Boolean Formulae to a learner. The target function we taught was the parity function, which can be quite difficult to learn from examples. Formally, given an input of size $n$, we define the target function to be the parity of a subset of $k$ input bits. We created a sample of $m$ examples, drawn uniformly at random from the space $\{0,1\}^n$. Our IMPACT algorithm for Boolean Functions was then provided different subsets of this sample via the teacher's filtering rule. To compare with learning algorithms developed for the standard supervised learning model, we also gave the full sample to various learning algorithms. Finally, from the same distribution, we created a test sample of $1,000$ examples. We ran this experiment for $10$ independent trials, and averaged over the results. We ran these experiments both on our algorithm and a set of algorithms selected from Weka's library~\citep{hall2009weka}. The algorithms we chose, to represent a variety of techniques and performances, were RandomTree, LADTree~\citep{Holmes2001}, BFTree~\citep{Shi2007,Friedman2000}, J48~\citep{Quinlan1993}, SimpleCart~\citep{Breiman1984}, AdaBoostM1~\citep{Freund1996}, and DecisionTable~\citep{Kohavi1995}.
\begin{figure}
\centering
\includegraphics[scale = .5]{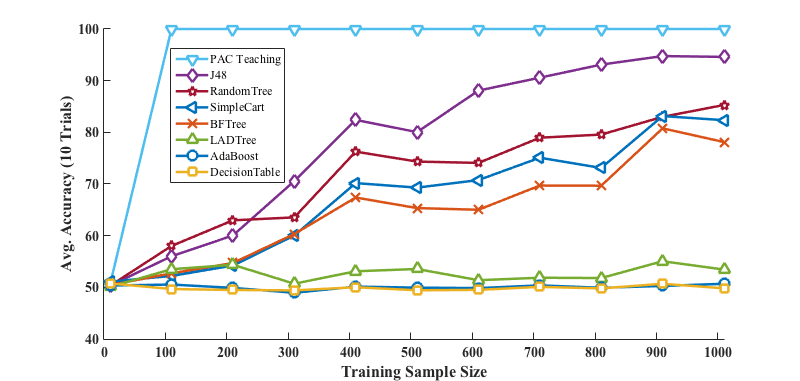}
\caption{Comparison of average accuracy over $10$ trials versus sample complexity for an IMPACT learner and standard PAC learners. The function used for this experiment was $f(x_0, \ldots, x_9) =  x_1 \otimes x_6 \otimes x_8 \otimes x_9$ with an input size of $n=10$.}
\label{fig:training}
\end{figure}
%
With a training sample size of $m=10$, all of the algorithms are unable to classify better than random guessing. However, once we begin to introduce some additional training data, the IMPACT algorithm quickly outperforms the standard supervised learners. In fact, once the training sample has reached roughly $100$ examples, our learner reliably achieves $100\%$ accuracy, while the standard learners are still within $10\%$ of random guessing. Figure~\ref{fig:training} shows the accuracy of each of these algorithms on the test sample for various training sample sizes.

\begin{figure}
\centering
\includegraphics[scale = .5]{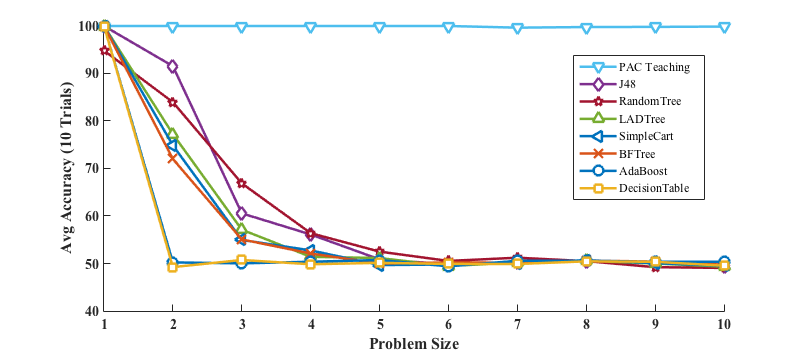}
\caption{Comparison of average accuracy over $10$ trials versus problem size, $k$, for our IMPACT learner and standard supervised learners. All input for this experiment was of size $n=10$.}
\label{fig:input}
\end{figure}

Additionally, we explored the effect of the size of the function to be learned, $k$, on our algorithm and the standard PAC algorithms. 
Figure~\ref{fig:input} shows how well each of the algorithms performed given $75$ training examples for problems of various sizes. The input was of size $n=10$, and the value on the x-axis shows $k$, how many bits of the input contributed to the parity function. All of the algorithms were able to correctly learn when the output was a direct copy of a single bit of input. However, once the function was classifying the parity of two bits, $75$ samples was insufficient for both AdaBoost and the DecisionTable. The rest of the standard algorithms reverted to having accuracy on par with random guessing by the time there were $7$ contributing variables. Our algorithm, however, remains above $95\%$ accuracy for all values we tested.
Choosing a larger sample size would have allowed some of the other algorithms to survive a bit longer. However, anything much larger than $75$ caused our algorithm to have $100\%$ accuracy for problem sizes that were much larger than we were considering.

\section{Conclusion}
In this work, we have developed a new learning framework that allows simple learners to learn complex and varied concepts. Our theoretical and empirical results were accomplished by introducing a knowledgeable, benevolent teacher to the framework, as well as by allowing the learner to build up knowledge across a series of discrete learning interactions. To illustrate the additional power of our new model, we presented a series of problems that were previously known to be unlearnable in the standard framework, given widely accepted cryptographic assumptions. We proved theoretically and verified experimentally that, under our model, these concepts are PAC Teachable. 

We recognize that the additional assumptions inherent in this model, that there is a teacher who knows the target concept and how to properly teach it, are non-trivial, and in some situations infeasible. However, by applying these ideas to other machine-learning fields like reinforcement learning, we think that there is potential for an even larger impact and that our assumptions are even more reasonable. Consider the problem of an end user training a household robot. It seems reasonable to assume that the end user, acting as the teacher, has a clear idea of the policy it wishes the robot to internalize. In settings like robotics, and in artificial intelligence more generally, a generalized version of the model we present could potentially be very powerful in creating competent agents. Doing so is one avenue we propose to explore in future work.

\newpage

\bibliographystyle{plain}
\bibliography{teaching}

\end{document}